\documentclass[final,12pt,cleveref,nameinlink]{conf2026} %

\title[Learning Conditional Averages]{Learning Conditional Averages}
\usepackage{times}
\usepackage{nicefrac}

\confauthor{%
	\Name{Marco Bressan} \Email{marco.bressan@unimi.it}\\
	\addr Università degli Studi di Milano, Italy
	\AND
	\Name{Nataly Brukhim}
	\Email{nbrukhim@princeton.edu}\\
	\addr{Institute for Advanced Study, USA}
	\AND 
	\Name{Nicolò Cesa-Bianchi} \Email{nicolo.cesa-bianchi@unimi.it}\\
	\addr Università degli Studi di Milano, Italy\\
	\addr Politecnico di Milano, Italy
	\AND
	\Name{Emmanuel Esposito} \Email{emmanuel@emmanuelesposito.it}\\
	\addr Università degli Studi di Milano, Italy
	\AND
	\Name{Yishay Mansour} \Email{mansour.yishay@gmail.com}\\
	\addr Tel Aviv University, Israel\\
	\addr Google Research
	\AND
	\Name{Shay Moran} \Email{smoran@technion.ac.il}\\
	\addr Technion, Israel\\
	\addr Google Research
	\AND
	\Name{Maximilian Thiessen} \Email{maximilian.thiessen@tuwien.ac.at}\\
	\addr TU Wien, Austria%
}

\usepackage{bm}
\usepackage{thmtools}
\usepackage{thm-restate}
\usepackage{framed}
\usepackage{algorithm}
\usepackage{algorithmic}

\usepackage{subcaption}

\usepackage[shortlabels]{enumitem}

\usepackage{mathabx} %

\renewcommand{\mathbf}{\boldsymbol}

\usepackage{mathrsfs} 

\renewcommand{\hat}{\widehat}

\def\X{{\mathcal X}}

\def\Y{{\mathcal Y}}

\def\E{{\mathbb E}}

\newcommand{\D}{\mathcal{D}}

\newcommand{\A}{\mathcal{A}}

\newcommand{\F}{\mathcal{F}}

\def\C{{\mathcal C}}

\def\bs{\mathbf{s}}

\newcommand{\ignore}[1]{}
\newcommand{\email}[1]{\texttt{#1}}

\DeclareMathAlphabet{\mathbfsf}{\encodingdefault}{\sfdefault}{bx}{n}

\let\Pr\relax
\DeclareMathOperator{\Pr}{\mathbb{P}}

\newcommand{\lrset}[1]{\left\{#1\right\}}

\renewcommand{\O}{\mathcal{O}}

\newcommand{\ind}[1]{\mathbb{I}\!\lrset{#1}}

\newcommand{\tvd}[1]{\left\lVert{#1}\right\rVert_{\text{TVD}}}

\newcommand{\eps}{\varepsilon}
\renewcommand{\epsilon}{\varepsilon}

\let\oldtfrac\tfrac
\renewcommand{\tfrac}[2]{\smash{\oldtfrac{#1}{#2}}}

\let\nablaold\nabla
\renewcommand{\nabla}{\nablaold\mkern-2.5mu}

\usepackage{bbm}

\newcommand{\Ev}{\mathcal{E}^{(i)}}

 \newcommand{\G}{\mathcal{G}}

\DeclareSymbolFontAlphabet{\mathbb}{AMSb}

\newcommand{\val}{\operatorname{V}}

\DeclareMathOperator{\scO}{\mathcal{O}}

\newcommand{\Ex}{\mathop{\mathbb{E}}}

\newcommand{\nat}{\mathbb{N}}

\usepackage{aliascnt}
\newaliascnt{appsec}{section}
\crefname{appsec}{appendix}{appendices}
\Crefname{appsec}{Appendix}{Appendices}
\aliascntresetthe{appsec}

\newcommand{\yD}{\widebar{c}} %

\newcommand{\mSmooth}{m}

\newcommand{\ymed}{y_{\mathrm{med}}}
\DeclareMathOperator{\vc}{VC}

\renewcommand{\tilde}{\widetilde}
\renewcommand{\hat}{\widehat}

\begin{document}

\maketitle

\begin{abstract}%
We introduce the problem of learning \emph{conditional averages} in the PAC framework.
The learner receives a sample labeled by an unknown target concept from a known concept class, as in standard PAC learning. However, instead of learning the target concept itself, the goal is to predict, for each instance, the average label over its \emph{neighborhood}---an arbitrary subset of points that contains the instance.
In the degenerate case where all neighborhoods are singletons, the problem reduces exactly to classic PAC learning. More generally, it extends PAC learning to a setting that captures learning tasks arising in several domains, including explainability, fairness, and recommendation systems. %
Our main contribution is a complete characterization of when conditional averages are learnable, together with sample complexity bounds that are tight up to logarithmic factors. The characterization hinges on the joint finiteness of two novel combinatorial parameters, which depend on both the concept class and the neighborhood system, and are closely related to the independence number of the associated neighborhood graph.
\end{abstract}

\begin{keywords}%
  supervised learning, PAC learning, graphs %
\end{keywords}

\section{Introduction}\label{sec:intro}
We introduce the following problem of learning conditional averages over neighborhoods.
We are given an arbitrary domain $\X$ and a concept class $\C\subseteq \{0,1\}^\X$. Each point $x \in \X$ in the domain has an associated \emph{neighborhood} $N[x]$, an arbitrary subset of $\X$ containing $x$ itself. For example,
if $\X$ is a metric space, then $N[x]$ could be a ball centered on $x$.
We can think of the neighborhoods as encoded by a directed graph $G=(\X,E)$, so that $(x,y) \in E$ whenever $y \in N[x]$.
Now, the learner receives a labeled sample $S=(x_1,c(x_1)),\dots,(x_m,c(x_m))$ from an unknown distribution $\D$, where $c\in\C$ is an unknown target concept.
The goal of the learner is to predict good estimates of the \emph{average} label in the neighborhood of $x$,
\begin{equation}
    \yD(x) = \E_{x'\sim\D}\bigl[c(x')\mid x'\in N[x]\bigr] \;.
\end{equation}
More precisely, like in standard PAC learning, we want the learner to output a predictor that, with probability $1-\delta$, has expected squared loss at most $\eps$ over $\D$.

This problem is interesting in that it captures tasks that appear in various settings related to fairness and explainability.
For example, in fairness we want that the average predictions (say, whether a person is hired or not) be roughly equal across certain subgroups \citep{hardt2016equality,rothblum2021multi,hsu2024distribution}.
We can encode these subgroups as neighborhoods (e.g., cliques) in the graph $G$.
In explainability, local explainers typically use neighborhoods around data points to describe the local behavior of a complex global classifier \citep{ribeiro2018anchors}.
Both scenarios require good estimates of the average predictions in each neighborhood, i.e., to estimate the conditional probabilities in (potentially small) regions.
A different type of applications can be envisioned in scenarios where classification is just an initial step in a pipeline of learning tasks where the final goal is to predict various aggregate statistics of the given population.
For example, given a sample of people and their income, we might want to predict average incomes across certain demographic subgroups, such as age, gender, or location.
Moreover, a certain level of anonymity might be required: while the machine learning system has access to the individual income data, users should only see the predicted averages.
Furthermore, in recommendation tasks such as collaborative filtering \citep{mnih2007probabilistic}, systems provide recommendations based only on partial knowledge of the user features. This corresponds to our learning problem of predicting the average label over neighborhoods, which here correspond to all users with the same values over a subset of features.
Another related learning problem is \emph{learning from label proportions}.
An important instantiation of this setting \citep{kuck2005learning,busa2023easy,busa2025nearly,brahmbhatt2023pac,li2024optimistic} can be viewed as complementary to ours: there, the learner only gets to see the average labels and must predict the true labels of the individual data points. %

Our problem looks interesting from a technical perspective, too.
To begin with, it encompasses and generalizes other standard learning tasks. For instance, when all the neighborhoods are singletons, $N[x]=\{x\}$, we recover standard supervised PAC learning.
When instead all neighborhoods are the full domain, $N[x]=\X$, the problem becomes estimating the average label under $\D$.
Interestingly, we can also show that our problem is at least as expressive as learning with partial concept classes (see \Cref{sec:graphpairs}).
On the other hand, the problem is interesting in that natural approaches and standard combinatorial parameters seem to fail.
For example, estimating $\yD$ via uniform convergence leads to suboptimal bounds and additional dependencies; see, e.g., \citet{grunewalder2018plug} and \citet{balsubramani2019adaptive}.
Moreover, when learning conditional averages, neither the VC dimension of the concept class nor the VC dimension of the neighborhood system succeed in capturing learnability.
In fact, they appear very far from doing so: there are cases where learning with the full class $\C=\{0,1\}^\X$ is possible, and (perhaps more surprisingly) cases where $\yD$ is not learnable even when the class consists of just a single known concept, $\C=\{c\}$.
The reason is that the average labels $\yD$ depend both on the concept $c$ and the distribution $\D$, and this intertwinement makes things subtly different from other, more standard settings.

\subsection{Main results}
Our main result is a full combinatorial characterization of learnability of conditional averages, for any concept class~$\C$ and neighborhood graph $G$, as well as bounds on the required sample complexity tight up to logarithmic factors. Interestingly, neither the class $\C$ nor the graph $G$ alone provide any meaningful characterizations.
We introduce two novel combinatorial parameters that depend on both $\C$ and $G$, and related to the independence number $\alpha(G)$ of $G$. The first one, $\alpha_1(G,\C)$, is the cardinality of the largest independent set in $G$ that is shattered by $\C$. In particular, $\alpha_1(G,\C)$ specializes to the standard VC dimension of $\C$ when neighborhoods are singletons, i.e., $G$ has no edges. 
The second parameter, $\alpha_2(G,\C)$, is the cardinality of the largest independent set $I$ for which there exists a concept $c\in \C$ such that every $x\in I$ has a neighbor in $G$ with opposite $c$-label.

We show that the finiteness of both these parameters is necessary and sufficient for learning. Moreover, we prove lower and upper bounds on the sample complexity $m(\eps, \delta)$ of any class $\C$ and graph $G$, of the following form (we omit the dependency of $\alpha_1,\alpha_2$ on $(G,\C)$ for brevity): 
\begin{equation} \label{eq:mainres}
\Omega\left(\frac{\alpha_1+\alpha_2/\log\alpha_2+\log(1/\delta)}{\eps}\right) \le m(\eps, \delta) \le \scO\left(\frac{\alpha_1+\alpha_2\log(1/\eps)}{\eps}\log\frac{1}{\delta}\right) \,,
\end{equation}
where $\eps$ is the accuracy  (w.r.t.\ the square loss) and $\delta$ the confidence parameter. 

Our learning algorithm is as follows: if a test point $x$ is adjacent to other points in the training sample, we return the empirical average of their labels. Otherwise, we run the one-inclusion graph algorithm \citep{haussler1994predicting} on all isolated points in the graph induced by the training set and the test point. This leads to an in-expectation guarantee, which we subsequently turn to a high-probability guarantee by amplifying the confidence through a median-based argument. %

\subsection{Further related work} \label{sec:further}
One motivation to study this particular learning problem comes from a line of work on the theory of explainable machine learning. %
\emph{Local explainers} are one of the primarily tools to study and explain the local behavior of a complex machine learning system. One well-known approach is that of \emph{anchors} \citep{ribeiro2018anchors}: to explain the label of a point, %
a small region around the point is used such that most points in the region have the same label. Many similar such local explanation methods exist \citep{ribeiro2016should,ancona2018towards}. This problem was formalized, among others, by \citet{dasgupta2022framework}. In a follow-up work, \citet{bhattacharjee2024auditing} studied the problem of \emph{auditing} local explanations, i.e., verifying whether the given explanation---the local classifier---has high accuracy in the chosen region. Our results are directly applicable to this auditing problem in the case of constant local classifiers (i.e., anchors). We also provide further support for their negative result: %
they state that typical local explainers require extremely small regions (e.g., exponentially small in the dimension of $\X=\mathbb{R}^d$) to be accurate, which leads to a large sample complexity to audit them. This corresponds in our case to graphs with large or infinite independence number $\alpha(G)$. For example, many small disjoint regions in $\mathbb{R}^d$ lead to a graph with large $\alpha(G)$ in our case. 

Another related work is the smoothed analysis of PAC learning by \citet{chandrasekaran2024smoothed}. Instead of having to predict averaged labels, their goal is to output a regular binary classifier that, however, competes with the best possible loss averaged over Gaussian perturbations of the test point, similar to our averages over neighborhoods. While they provide a novel (smoothed) analysis of standard binary classification, we tackle a novel learning problem instead.

\section{Learning average labels}

We start with the graph-theoretic notation used in the rest of the paper.
We denote by $G = (V,E)$ a simple directed graph with anti-parallel edges allowed.
That is, for any distinct $u,v \in V$, both $(u,v)$ and $(v,u)$ may belong to $E$ (but no parallel edges are allowed).
We remark that $V$ can be infinite, and it will also be convenient to think of $E$ as a relation over $V$.
The \emph{out-neighborhood} of $x \in V$ is $N^+(v)=\{v'\in V \mid (v, v')\in E\}$, and the \emph{in-neighborhood} is $N^-(v)=\{v'\in V \mid (v',v)\in E\}$.
We may also use $N(v) = N^+(v)$ for brevity.
The \emph{closed} out-neighborhood of $v \in V$ is the set $N[v]=N^+(v)\cup\{v\}$; we often refer to it simply as the neighborhood of $v$.
A subset $V' \subseteq V$ is independent in $G$ if for every distinct $u,v \in V'$ we have $u \notin N^+(v) \cup N^-(v)$.
The independence number $\alpha(G)$ of $G$ is the cardinality of its largest independent set (note that we may have $\alpha(G)=\infty$).

We now define the setting of the learning problem. 
Let $\X$ be a non-empty instance space and $\C\subseteq \{0,1\}^\X$ a concept class.
We assume a fixed neighborhood structure over $\X$ encoded by a simple directed graph $G=(\X, E)$, with anti-parallel edges allowed.
The learner does not need to know $G$ in advance; it is sufficient that it can access $G$ through an edge oracle, i.e., an oracle that given $(u,v) \in V^2$ tells whether $(u,v) \in E$. As in realizable PAC learning, the learner receives a labeled sample $S=\bigl((x_1,c(x_1)),\dots,(x_m,c(x_m))\bigr)$ with each $x_i$ i.i.d.\ from an unknown distribution $\D$ and where $c\in\C$ is the unknown target concept. In contrast to standard PAC learning, our goal is to predict the \emph{average labels} over each (closed) neighborhood:\footnote{We assume throughout that the
    directed edge relation of the graph is measurable, as well as any additional measurability assumptions that are required by the one-inclusion graph algorithm \citep{haussler1994predicting}. 
    }
\begin{equation}
    \yD(x) = \E_{x' \sim \D}\Bigl[c(x') \mid x' \in N[x]\Bigr]\,,
\end{equation}
where we suppress the dependence of $\yD$ on both the underlying graph $G$ and distribution $\D$ when it is clear from context. 
In particular, the goal is to design a learning rule $\A$ that given a sample $S$ outputs a predictor $h : \X \to [0,1]$ such that with probability at least $1-\delta$,
\begin{equation}\label{eq:loss}
    L(h) = \E_{x \sim \D} \left[\bigl(h(x) - \yD(x)\bigr)^2\right] \le \eps \,,
\end{equation}
which correspond to the risk of $h$ with respect to the square loss.
For convenience, we will sometimes use the loss notation $\ell_x(h) = \bigl(\yD(x) - h(x)\bigr)^2$ for individual points $x$. We remark that other natural loss functions are possible; %
see further discussion in \Cref{sec:discussion}.

Observe that, although the labels determined by the concept $c$ are binary, rather than a  classification problem our learning problem is a regression task. We define it formally below. 
\begin{definition}[$(G,\C)$-learner]
Let $\C\subseteq\{0,1\}^\X$ and let $G=(\X,E)$ be a directed graph. A learning rule $\A$ is a \emph{$(G,\C)$-learner} if there exists a sample complexity $m:(0,1)^2\to\nat$ such that for every distribution $\D$ over $\X$, every $c\in \C$, and every $\eps,\delta\in(0,1)$, the following holds. When given a multiset $S$ of $m \ge m(\eps,\delta)$ i.i.d.\ examples generated by $\D$ and labeled by $c$, then the learning rule returns a predictor $h_S = \A(S)$ such that $L(h_S) \le \eps$ with probability at least $1-\delta$ over $S$. 
\end{definition}

We now define two combinatorial parameters that characterize the learnability of our problem.  
The first can be thought of as corresponding to a \textit{shattered independent set}, and the second to a \textit{bichromatic independent set}. Formally, these parameters are defined as follows.
\begin{definition}[Parameters $\alpha_1$ and $\alpha_2$]\label{def:params}
Let $G=(\X,E)$ be a directed graph, $\C\subseteq\{0,1\}^\X$ a class, and $c\in\C$ a concept.
\begin{itemize}[leftmargin=*]
  \setlength\itemsep{0em}
    \item \textbf{Largest shattered IS.} Denote by $\alpha_1(G,\C)$ the size of a largest independent set $I\subseteq \X$ of $G$ that is shattered by $\C$, that is, $|\{c\cap I : c\in\C\}|=2^{|I|}$. 
    \item \textbf{Largest bichromatic IS.} Denote by $\alpha_2(G,c)$ the size of a largest independent set $I\subseteq\X$ of $G$ such that each vertex in $I$ has a neighbor in $G$ with opposite label with respect to $c$. Define $\alpha_2(G,\C)=\sup_{c\in\C}\alpha_2(G,c)$.
\end{itemize}
\end{definition}
 We remark that we can equivalently define $\alpha_2(G,c)$ in the following way.
Let $X_c\subseteq \X$ be the set of vertices $x$ that have a neighbor $x' \in N(x)$ in $G$ such that $c(x)\neq c(x')$. Denote by $G_c$ the subgraph of $G$ induced by $X_c$. Then, we have $\alpha_2(G,c)=\alpha(G_c)$.\\

\noindent 
Our main result is the following characterization of learnability.
\begin{theorem}\label{thm:characterization}
      Let $\C\subseteq\{0,1\}^\X$ and let $G=(\X,E)$. Then:
      \[
      \text{There exists a }(G,\C)\text{-learner } \iff \alpha_1(G,\C) + \alpha_2(G,\C)<\infty \;.
      \]
\vspace*{.5em}
\end{theorem}

\section{Examples and special cases}
Before turning to the proof of our main result, we discuss some special cases to illustrate the role of the two parameters $\alpha_1$ and $\alpha_2$.
Recall that, in the degenerate case where $G$ has no edges, the problem coincides with standard PAC learning.
In particular, if $G$ has no edges, then for every concept class $\C$ we have $\alpha_2(G,\C) = 0$, while  $\alpha_1(G,\C)$ becomes just the VC dimension of $\C$. 

The following are two extreme examples of concept classes and corresponding corollaries of \Cref{thm:characterization}. 
In a nutshell, \Cref{cor:full,cor:singleton} show that the VC dimension of $\C$ alone does not determine learnability in our model.
The first case we consider is the full concept class $\C=\{0,1\}^\X$.
Since any independent set in the graph $G$ is shattered by $\C$, we have $\alpha_1(G,\C)=\alpha(G)$.
Moreover, as $\alpha_2(G,\C)\le \alpha(G)$ by definition, from \Cref{thm:characterization} we obtain the following corollary.
\begin{corollary}[Full class]\label{cor:full} \!\!Let $G=(\X,E)$.
    There is a $(G,\{0,1\}^\X)$-learner if and only if $\alpha(G)<\infty$. 
\end{corollary}
Thus, even on the full class $\C=\{0,1\}^{\X}$, one may still be able to learn conditional averages, depending on the graph $G$.
For a concrete example, take $G$ to be the full graph, so that $N[x]=\X$ for every $x \in \X$.
In this case $\alpha(G)=1$, so by \Cref{cor:full} a $(\C,G)$-learner exists; and, indeed, in this case $\yD(x) = \E_{x \sim D}c(x)$ for every $x$, so the problem is just learning the average label under $\D$.

The other extreme case is the class consisting of a single concept, i.e., $\C=\{c\}$ for some $c: \X \to \{0,1\}$. As in this case $\alpha_1(G,\{c\})=0$, \Cref{thm:characterization} yields the next corollary.
\begin{corollary}[Singleton class]\label{cor:singleton}
Let $G=(\X,E)$ and $c:\X\to\{0,1\}$.    There exists a $(G,\{c\})$-learner if and only if $\alpha_2(G,c)<\infty$.
\end{corollary}
Thus, learning conditional averages can be hard even when the learner knows the target concept $c$.
The reason is that, even with $c$ known, learning $\yD$ requires one to estimate the average labels in each neighborhood, and those depend on the (unknown) distribution.

\subsection{Special neighborhood graphs}
We now turn to special cases of graphs that illustrate the role of the graph structure in our task. 

\paragraph{Complete graphs.}
If $G=(\X,E)$ is a complete graph, that is, $N[x]=\X$ for all $x\in\X$, then our learning problem simply requires to estimate the average label $\E_{x\sim\D}[c(x)]$ under the distribution $\D$ and ground truth $c$. In this case, $\alpha_1(G,\C),\alpha_2(G,\C)\le \alpha(G)=1$ as there are no larger independent sets in $G$. The problem is learnable---independently of the complexity of $\C$---by simply predicting the empirical mean of labels using a sample of size $\scO(1/\eps)$.
\paragraph{Tournament graphs.} Similarly, on tournament graphs, which are orientations of complete (undirected) graphs, learning is easy. By definition we again have $\alpha_1(G,\C),\alpha_2(G,\C)\le \alpha(G)=1$ and thus we only need  a sample of size $\widetilde\scO(1/\eps)$ even for $\C=\{0,1\}^\X$. 
This is somewhat surprising, since unlike the complete graph case, neighborhoods in a tournament graph can overlap in complex ways. Specifically, viewing neighborhoods as a set system, the \emph{VC dimension of the set system} in tournament graphs can be as large as $\log|\X|$, whereas in the complete graph case it is exactly $1$. These examples show that neighborhood complexity does not determine learnability.\\

In both of the examples above, the parameter $\alpha_2$ is small. In contrast, a simple example of a graph $G$ with a \textit{large} $\alpha_2$ is a  star graph: a tree graph with one central root node and all other nodes as leaves. Then, let $c$ be a concept that assigns the same label to all leaves, and the opposite label to the root. In this case $\alpha_2(G,c)$ will correspond to the number of leaves.

\section{Main results} \label{sec:main}

In this section we describe our algorithmic approach to the conditional average  learning problem, followed by the analysis of upper and lower bounds on the sample complexity of a $(G,\C)$-learner.

\subsection{Upper bound}
We first present our learning algorithm, described in \Cref{alg:main} below.
Then, in  \Cref{thm:UBexpectation} we give a bound on the expected error of \Cref{alg:main}, and in \Cref{thm:highProbUB} we give the high-probability bound.
The following \Cref{alg:main} relies on the classic One-Inclusion Graph (OIG) algorithm as a sub-routine (see \Cref{app:oi} for the algorithm and further details). 
\begin{algorithm}[h!]\SetAlgoNoEnd
\caption{Conditional Average Learning Algorithm for Graphs}
\label{alg:main}
\textbf{Input}: Training set \( S = \{(x_i, y_i)\}_{i=1}^m \), test point \( x \).
\BlankLine

Let $\hat G$ be the subgraph of $G$ induced by $\{x_1,\dots,x_m,x\}$.\;

\uIf{$x$ has a neighbor in $\hat G$}{
    \textbf{Return} the empirical fraction of neighbors with label $1$ in $N_{\hat G}[x]$.\;
}\Else{
    Let $I$ be the set of isolated vertices of $\hat G$.\;
    
    Run the OIG predictor (see \Cref{algo:one_inc}) on $(I, \C|_I$) and \textbf{return} its prediction.\,
}    
\end{algorithm}

 \begin{theorem}\label{thm:UBexpectation}
     Let $\C\subseteq\{0,1\}^\X$ and let $G=(\X,E)$ be a directed graph with $\alpha_1(G,\C)+\alpha_2(G,\C)<\infty$. Then, for every distribution $\D$ over $\X$, every $c\in \C$ and every $\eps\in(0,1)$, if $S$ is an i.i.d.\ sample over $\D$ of size $m \ge m(\eps) = \O\left(\frac{\alpha_1(G,\C)+\alpha_2(G,\C)\log(1/\eps)}{\eps}\right)$ then $\E_{S\sim \D^m}[L(h_S)] \le \eps$, where $h_S = \A(S)$ is defined by applying \Cref{alg:main} over $S$ and an input point $x$. 
\end{theorem}

The proof of this result relies on the following two technical lemmas.
The first lemma controls the loss of the predictor returned by \Cref{alg:main} over all points whose neighborhood has sufficiently large mass.
In turn, we use this result to determine the sample size required to guarantee that such a loss is sufficiently small in expectation.
\begin{restatable}{lemma}{heavyVertexLemma}\label{lem:heavyVertex}
  Let $\lambda\in(0,1]$, $c: \X \to \{0,1\}$, $G=(\X,E)$, $\D$ a distribution over $\X$, and $S$ a random multiset of $m \in \nat_+$ i.i.d.\ samples from $\D$ and labeled by~$c$.
  Let $h_S$ be the output of \Cref{alg:main} given input $S$.
  Then, all points $x \in \X$ with $\D(N[x]) \ge \lambda$ satisfy
  $\E_{S \sim \D^m} \bigl[\ell_x(h_S)\bigr] \le \frac{7}{2m\lambda}$.
\end{restatable}
  
The proof of \Cref{lem:heavyVertex} is in \Cref{app:proofs-ub}.
We remark that the only part where we explicitly use the squared loss is \Cref{lem:heavyVertex}.
The whole approach can easily be adapted to other loss functions, by simply modifying this particular aspect; see \Cref{sec:discussion} for an example.
The second lemma, while simple, is crucial in enabling our result.
\Cref{lem:lightBichromaticVertex} bounds the total mass of points whose neighborhoods are light. 
\begin{lemma}[Light-neighborhood nodes are light]\label{lem:lightBichromaticVertex}
      Let $G=(V,E)$ be a (possibly infinite) directed graph and $\D$ a distribution over $V$ such that $E$ is measurable. 
      Then, for any $\lambda \in [0,1]$, 
      \[
      \D\left(\{v : \D(N[v])\le \lambda\}\right) \le 2\lambda \cdot \alpha(G) \;.
      \]
\end{lemma}
\begin{proof} Let $V_\lambda$ denote the subset of nodes $v$ of $G$ for which $\D(N[v]) \le \lambda$, and consider the induced subgraph over $V_\lambda$, denoted $G_\lambda$.
We would like to bound $\D(V_\lambda)$ from above.
To that end, we will iteratively remove nodes from $V_\lambda$ and bound the overall mass of removed points.
By \Cref{clm:degree_symm}, there must exist $v \in V_\lambda$ with $\D\left(N^+(v)\right) \;\ge\; \D\left(N^-(v)\right)$, where the neighborhoods $N^+,N^-$ are with respect to $G_\lambda$.
Remove all nodes in $N^+(v) \cup N^-(v)$ and $v$ itself from $V_\lambda$.
Observe that the overall mass of nodes we have removed is at most $2\lambda$.
We repeat this process until there are no more vertices left in $V_\lambda$. Let $S$ denote the set of nodes $v$ we picked in each round. 
Observe that $S$ forms an independent set in $G_\lambda$, and so $|S| \le \alpha(G_\lambda) \le \alpha(G)$ bounds the total number of rounds in the process.
Overall, we obtain that $\D(V_\lambda) \le 2\lambda\cdot\alpha(G)$ as claimed. 
\end{proof}

We remark that in the undirected case, \Cref{lem:lightBichromaticVertex} is related to a weighted version of the Caro-Wei inequality \citep{caro1979independence,wei1981stability}; our analogous result, which applies more generally to possibly uncountable domains and arbitrary (measurable) edge sets, might be of independent interest.
Its proof relies on the following result.
\begin{proposition}\label{clm:degree_symm}
    Let $G = (V, E)$ be a (possibly infinite) directed graph and let $\D$ be a probability distribution on $V$ such that the edge relation $E$ is measurable. 
    Then, there exists a vertex $v \in V$ such that $\D\left(N^+(v)\right) \ge \D\left(N^-(v)\right)$.
\end{proposition}
\begin{proof}
    We start by showing that we have $ \mathbb{E}_{v\sim \D}[\D\left(N^+(v)\right)] =  \mathbb{E}_{v\sim \D}[\D\left(N^-(v)\right)]$, which corresponds to a weighted version of the degree sum formula for directed graphs.
Let $x,x'$ be independent random variables distributed according to $\D$.
Define the indicator random variable $I_E \coloneq \ind{(x, x') \in E}$. Conditioning on $x$, we have
\begin{equation}
    \E[I_E\mid x]
    = \Pr\bigl((x,x') \in E \mid x\bigr)
    = \Pr\bigl(x' \in N^+(x)\mid x\bigr)
    = \D\bigl(N^+(x)\bigr) \;.
\end{equation}
By the law of total expectation (which applies since the edge relation is measurable), we have $\E[I_E]
= \E\!\left[\E[I_E\mid x]\right]
= \E_{x \sim \D}\bigl[\D(N^+(x))\bigr]$.
Similarly, conditioning on $x'$, we have
\begin{equation}
    \E[I_E\mid x']
    = \Pr\bigl((x,x') \in E \mid x'\bigr)
    = \Pr\bigl(x \in N^-(x')\mid x'\bigr)
    = \D\bigl(N^-(x')\bigr) \;.
\end{equation}
Again by the law of total expectation, $\E[I_E]
= \E\!\left[\E[I_E\mid x']\right]
= \E_{x \sim \D}\big[\D(N^-(x))\big]$. Thus, we have $ \mathbb{E}_{v\sim \D}[\D\left(N^+(v)\right)] =  \mathbb{E}_{v\sim \D}[\D\left(N^-(v)\right)]$.
Next, if $\D\left(N^+(v)\right) < \D\left(N^-(v)\right)$ held for all $v \in V$, then we would have $\mathbb{E}_{v\sim \D}[\D\left(N^+(v)\right)] < \mathbb{E}_{v\sim \D}[\D\left(N^-(v)\right)]$, in contradiction with the above. Hence, there must exist $v \in V$ such that $\D\left(N^+(v)\right) \ge \D\left(N^-(v)\right)$.
\end{proof}

Finally, we have all the main tools to prove the upper bound on the sample complexity for achieving the in-expectation guarantee on the loss.

{
\renewcommand{\proofname}{Proof of \Cref{thm:UBexpectation}}
\begin{proof}
    Let $c\in\C$ be the target concept.
    Let $\alpha_1=\alpha_1(G,\C)$ and $\alpha_2=\alpha_2(G,\C)$, and let $m \ge m_{\A}(\eps)$ where the constants in the $\O(\cdot)$ notation are large enough.
    Define $\eps'=\frac{\eps}{3 \alpha_2}$.
    We consider a learner $\A$ that gets a training set $S$ as input and outputs a predictor $h_S: \X \to [0,1]$ by applying \Cref{alg:main} over $S$ and the input point $x \in \X$. For any $c \in \C$ and distribution $\D$ over $\X$ denote by 
    $\D_c$ the joint distribution over $(x,y) \in \X\times \{0,1\}$ determined by $x \sim \D$ and $y = c(x)$. With some abuse of notation, we will omit $c$ when it is clear from context.     Then, note that rather than bounding the expected error of $\A$, we can instead analyze its leave-one-out error:
    \begin{align}
        \E_{S \sim \D^m}[L(\A(S))]
        = \Ex_{S \sim \D^m, x \sim \D}[\ell_x(\A(S))] 
        = \Ex_{S' \sim \D^{m+1}}\Bigl[\Ex_{i \sim \mathcal{U}_{[m+1]}}\bigl[\ell_{x_i}(\A(S'_{-i}))\bigr]\Bigr] \;.
    \end{align}
    That is, we consider the following process: (i) a sample $S'$ of $m+1$ i.i.d.\ points is drawn from $\D^{m+1}$, (ii) an integer $i$ is drawn uniformly from $[m+1]$, (iii) the algorithm $\A$ is ran with $S = S'_{-i} = S' \setminus \{(x_i,c(x_i)\}$ as (labeled) training sample, and $x=x_i$ as test point, incurring error $\ell_x(h_S)$ where $h_S=\A(S)=\A(S'_{-i})$.
    Our goal is to show that: $\E_{S',i}[\ell_{x_i}(h_S)] \le \eps$,
    where it is understood that $S' \sim \D^{m+1}$ and $i$ is uniform over $[m+1]$.
   We consider the following events (where $N = N^+_G$):
    \begin{align}
        \Ev_1 &: N[x_i] \text{ is bichromatic},\\
        \Ev_2 &: N[x_i] \text{ is monochromatic and } x_i \notin I, \text{ and} \\
        \Ev_3 &: N[x_i] \text{ is monochromatic and } x_i \in I. 
    \end{align}
    One can check that the three events are disjoint and that $\Ev_1 \cup \Ev_2 \cup \Ev_3$ corresponds to the sample space.
    Thus, by the law of total expectation,
    \begin{align}
        \E_{S',i}\left[\ell_{x_i}(h_S)\right]= \sum_{r=1}^3 \Pr_{S',i}(\Ev_r) \cdot \E_{S',i}\!\left[\ell_{x_i}(h_S) \,|\, \Ev_r\right] \;.
    \end{align}
    In what follows, we show that each term of the summation in the right-hand side is sufficiently small; precisely, the terms corresponding to $\Ev_1$ and $\Ev_3$ are each bounded from above by $\frac{\eps}{2}$, while the one relative to $\Ev_2$ is equal to zero.

    \paragraph{Loss under $\Ev_1$.}
    The idea is to bound the error by simultaneously applying \Cref{lem:heavyVertex,lem:lightBichromaticVertex}.
    However, their respective guarantees have an inverse dependence on the range of possible values that $\D(N[x_i])$ can take, and thus a straightforward application of these two technical results can lead to a larger guarantee than desired, with a worse dependence on $\alpha_2$ or $m$.
    We show this can be avoided by carefully defining a partitioning of the range of possible values for $\D(N[x_i])$ into a sufficiently small number of intervals determined by a geometric sequence.
    
    Let $n = \lceil\log_2(m/\alpha_2)\rceil+1$ and define $1 \ge b_1 \ge \dots \ge b_n \ge 0$ to be a non-increasing sequence such that $b_j = 2^{-j+1}/\alpha_2$ for all $j \in [n]$.
    Note that this sequence satisfies $b_1=1/\alpha_2$ and $b_n \le 1/m$.
    For every $j\in [n]$, denote by $A_j$ the event that $\D(N[x_i]) \le b_j$, and define $B_j = A_j \setminus A_{j+1}$. By a further application of the law of total expectation, and using the fact that $\ell_x(h) \in [0,1]$ for every $x \in \X$, we can bound the error conditioned on $\Ev_1$ as
    \begin{align}
        \E_{S',i}\bigl[\ell_{x_i}(h_S)\mid \Ev_1\bigr]
        &\le \sum_{j=1}^{n-1} \E_{S',i}\bigl[\ell_{x_i}(h_S)\mid \Ev_1, B_j\bigr] \cdot \Pr_{S',i}(B_j \mid \Ev_1) \\
        &\quad + \Pr_{S',i}(A_n \mid \Ev_1) + \E_{S'_i}\bigl[\ell_{x_i}(h_S) \mid \Ev_1, \overline A_1\bigr]\;.
    \end{align}
    For each $j \in [n-1]$, by symmetrization and \Cref{lem:heavyVertex} we have that
    \begin{align}
        &\E_{S',i}\bigl[\ell_{x_i}(h_S)\mid \Ev_1, B_j\bigr] \\
        &\quad = \E_{x \sim \D}\Bigl[\E_{S\sim \D^m}\bigl[\ell_{x}(h_S) \bigr] \mid N[x] \text{ is bichromatic}, \D(N[x]) \in (b_{j+1}, b_j]\Bigr] \\
        &\quad \le \sup_{x : \D(N[x]) \ge b_{j+1}} \E_{S \sim \D^m}\bigl[\ell_{x}(h_S)\bigr]
        \le \frac{7}{2b_{j+1}m} \;,
    \end{align}
    where we rely on the fact that $S$ and $x_i$ are independent, and both events $\Ev_1$ and $B_j$ only affect $x_i$.
    By a similar reasoning, it also holds that $\E_{S',i}\bigl[\ell_{x_i}(h_S) \mid \Ev_1, \overline A_1\bigr] \le \frac{7}{2b_1m}$.
    Moreover, we have that $\Pr_{S',i}(B_j \mid \Ev_1) \le \Pr_{S',i}(A_j \mid \Ev_1)$ as $B_j \subseteq A_j$.
    These observations allow us to show that
    \begin{equation}
        \E_{S',i}\bigl[\ell_{x_i}(h_S)\mid \Ev_1\bigr]
        \le \frac{7}{2m} \sum_{j=1}^{n-1} \frac{\Pr_{S',i}(A_j\mid \Ev_1)}{b_{j+1}} + \Pr_{S',i}(A_n \mid \Ev_1) + \frac{7}{2b_1m} \;,
    \end{equation}
    and consequently that
    \begin{align}
        \Pr_{S',i}(\Ev_1) \E_{S',i}\bigl[\ell_{x_i}(h_S)\mid \Ev_1\bigr]
        &\le \frac{7}{2m} \sum_{j=1}^{n-1} \frac{\Pr_{S',i}(A_j)}{b_{j+1}} + \Pr_{S',i}(A_n) + \frac{7}{2b_1m} \\
        &\le \frac{7\alpha_2}{m} \sum_{j=1}^{n-1} \frac{b_j}{b_{j+1}} + 2b_n\alpha_2  + \frac{7}{2b_1m} \\
        &\le\frac{14\alpha_2}{m}(n-1) +  \frac{2\alpha_2}{m} + \frac{7\alpha_2}{2m} \\
        &\le \frac{2\alpha_2}{m}(7\log_2(m/\alpha_2) + 11) \;,\label{eq:lastE1}
    \end{align}
    where the second step holds because $\Pr_{S',i}(A_j) \le 2b_j\alpha_2$ for each $j \in [n]$, as we show next, by applying \Cref{lem:lightBichromaticVertex}.
    Using \Cref{lem:lightBichromaticVertex} over the entire $G$ would lead to a worse upper bound of $2b_j\alpha(G)$. Instead, we apply it to the subgraph $G_c$ induced by all nodes with bichromatic neighborhoods $N[x]$ with respect to $c$. As remarked before (after \Cref{def:params}), $\alpha(G_c)=\alpha(G,c)\le \alpha_2$.
    
    Then, the right-hand side of \Cref{eq:lastE1} is no larger than $\frac{\eps}{2}$ given $m \ge \scO\bigl(\frac{\alpha_2}{\eps}\log(1/\eps)\bigr)$.
    
    \paragraph{Loss under $\Ev_2$.}
    Here, the algorithm predicts the empirical average of the labels of $S \cap N[x_i]$, which is the true label $\yD(x_i) =c(x_i)$ since $N[x_i]$ is monochromatic.
    Thus $\E_{S',i}\bigl[\ell_{x_i}(h) \mid \Ev_2\bigr] = 0$.

    \paragraph{Loss under $\Ev_3$.}
    By non-negativity of the loss $\ell_{x_i}(h)$, we have 
    \begin{align}
         \Pr_{S',i}(\Ev_3) \cdot \E_{S',i}\!\left[\ell_{x_i}(h) \,|\, \Ev_3\right]\le \Pr_{S',i}(x_i \in I) \cdot \E_{S',i}\!\left[\ell_{x_i}(h) \,|\, \Ev_3\right]\,.
    \end{align}
    We now proceed to show that $\Pr_{S',i}(x_i \in I) \cdot \E_{S',i}\bigl[\ell_{x_i}(h) \mid \Ev_3\bigr] \le \frac{\eps}{2}$ for \emph{every} realized sequence $S' \in \X^{m+1}$ with $m$ large enough.
    First, since $i$ is uniform over $[m+1]$, then
    \begin{align}
        \Pr_{S',i}(x_i \in I) \cdot \E_{S',i}\!\left[\ell_{x_i}(h) \,|\,\Ev_3\right] =
         \frac{|I|}{m+1} \cdot \E_{S',i}\!\left[\ell_{x_i}(h) \,|\, \Ev_3\right]\,,
         \label{eq:Ev4}
    \end{align}
     Note that in this case we have  
     $\yD(x_i) =  c(x_i)$ as $N[x_i]$ is monochromatic. Thus, for all such $x_i$ the learning task is a binary classification with the zero-one loss (instead of regression with the square loss).
    Now, if $x_i \in I$, then the distribution of $x_i$ is uniform over $I$ and the algorithm returns the OIG prediction for $x_i$.
    This implies that $\E_{S',i}\bigl[\ell_{x_i}(h) \mid \Ev_3\bigr]$ %
    is the expected leave-one-out error of the OIG predictor ran on $I$.
    As the vertices in $I$ are isolated in $\hat{G}$, $I$ is an independent set in $\hat G$, and so $I$ is independent in $G$ too. Therefore, the projection $\C_{|I}$ of $\C$ on $I$ has VC dimension at most $\alpha_1$, by definition of $\alpha_1$. The OIG predictor \citep{haussler1994predicting} thus yields $\E_{S',i}\bigl[\ell_{x_i}(h) \mid \Ev_3 \bigr] = \frac{\alpha_1}{|I|+1}$.
        
    As a consequence, \Cref{eq:Ev4} satisfies
    \begin{align}
        \Pr_{S',i}(x_i \in I) \cdot \E_{S',i}\!\left[\ell_{x_i}(h) \,|\,\Ev_3\right] \le \frac{|I|}{m+1}\cdot \frac{\alpha_1}{|I|+1} < \frac{\alpha_1}{m+1}\,,
    \end{align}
    which, as $m\ge\O(\frac{\alpha_1}{\eps})$, is at most $\frac{\eps}{2}$. 
    This concludes the proof.
\end{proof}
}

As a final step, we adopt a confidence amplification approach to turn the in-expectation guarantee of \Cref{thm:UBexpectation} into a high-probability one.
\begin{theorem}\label{thm:highProbUB}
     Let $\C\subseteq\{0,1\}^\X$ and let $G=(\X,E)$ be a directed graph with $\alpha_1(G,\C)+\alpha_2(G,\C)<\infty$. There exists a $(G,\C)$-learner with overall sample complexity
     \[
        \mSmooth(\eps,\delta)=\scO\left(\frac{\alpha_1(G,\C)+\alpha_2(G,\C)\log(1/\eps)}{\eps}\log(1/\delta)\right)\,.
    \]
\end{theorem}
\begin{proof}
    The proof follows by applying \Cref{thm:UBexpectation} to show that \Cref{alg:main} has a bounded expected error, and then applying \Cref{lemma:highProbUB} to obtain a high-probability guarantee.
\end{proof}

The high-probability guarantee is obtained by using the following lemma.
The main idea is to use Markov's inequality to obtain a guarantee with probability strictly larger than $1/2$ from the in-expectation one provided by \Cref{thm:UBexpectation}, and then take the pointwise median over $\O(\log(1/\delta))$ independently obtained predictors from \Cref{alg:main} to amplify the success probability.
\begin{lemma}\label{lemma:highProbUB}
    Let $\A$ be an algorithm that satisfies an expected error $\E_S[L(\A(S))]\le \eps$ with sample size $m(\eps)$.
    There exists a $(G,\C)$-learner $\A_{\mathrm{whp}}$ with oracle access to $\A$ and overall sample complexity $m(\eps,\delta) = \scO\bigl(m(\eps/100) \log(1/\delta)\bigr)$.
\end{lemma}
\begin{proof}
    Let $k\in\nat$ and $\eps'\in(0,1)$ to be determined later. Let $S$ be an i.i.d.\ sample from $\D$ of size $k\cdot m(\eps)$ and partition $S$ into subsamples $S_1,\dots,S_k$ each of size $m(\eps)$. 
    Let $h_i=\A(S_i)$ for each $i\in[k]$ be the predictor satisfying $\E_{S_i}[L(h_i)]\le\eps'$.
    Denoting by $A_i$ the event $L(h_i)> 10\eps'$, we have by Markov's inequality %
    $\E_{S_i}[\ind{A_i}]=\Pr_{S_i}(A_i)\le 1/10$.
    By a multiplicative Chernoff bound, there is an absolute constant $c > 0$ such that $\Pr_{S}\left(\sum_{i=1}^k \ind{A_i}\ge k/5 \right) \le e^{-ck}$.
    The right-hand side of the latter is at most $\delta$ for $k=\bigl\lceil(1/c)\log(1/\delta)\bigr\rceil$.
    Thus, with probability at least $1-\delta$, we have $\frac45 k$ classifiers $h_i$ satisfying $L(h_i)\le 10\eps'$.
    We now apply \Cref{lemma:median} from \Cref{app:proofs-ub} and see that the pointwise median $h_{\mathrm{med}}=\operatorname{median}(h_1,\dots,h_k)$ satisfies $L(h_{\mathrm{med}}) \le 70\eps'\le 100\eps'$ with probability $1-\delta$.
    We define $\A_{\mathrm{whp}}$ as an algorithm that takes a sample $S$ a described and returns $h_{\mathrm{med}}$.
    The claim follows by choosing $\eps'=\eps/100$.
\end{proof}

\subsection{Lower bounds}
We prove a separate lower bound for each parameter $\alpha_1(G,\C)$ and $\alpha_2(G,\C)$, which together imply tightness  of our algorithmic result from \Cref{thm:highProbUB} up to log-factors.
\begin{lemma}%
Let $\C\subseteq\{0,1\}^\X$ and $G=(\X,E)$.
If $\alpha_1(G,\C)\ge 2$, then the sample complexity achievable by any $(G,\C)$-learner is 
$\mSmooth(\eps,\delta)=\Omega\left(\frac{\alpha_1(G,\C)+\log(1/\delta)}{\eps}\right)$.
\end{lemma}
\begin{proof}
Let $I \subseteq \X$ with $|I|\ge 2$ be an independent set that is shattered by $\C$ and $c\in\C$ a target concept.
As the neighborhoods of the vertices in $I$ are disjoint, any distribution with support on $I$ has $\yD(x)=c(x)$ for all $x\in I$.
Since $I$ is shattered by $\C$, we can apply PAC lower bounds to get a $\Omega(|I|/\eps)$ lower bound \citep{ehrenfeucht1989general}.
Further, we get an additional $\Omega\bigl(\log(1/\delta)/\eps\bigr)$ lower bound.
This follows again from PAC lower bounds, as by $\alpha_2(G,c)\ge 2$ there are at least $4$ different concepts in the projection of $\C$ onto $I$ and thus we can apply, e.g., \citet{blumer1989learnability}. %
\end{proof}

\begin{lemma}\label{lem:LB2}%
    Let $\C\subseteq\{0,1\}^\X$ and $G=(\X,E)$. Then, the sample
complexity  achievable by any $(G, \C)$-learner is  $\mSmooth(\eps, 1/2)=\Omega\left(\frac{\alpha_2(G,\C)}{\eps\log(\alpha_2(G,\C))}\right)$.%
\end{lemma}
\begin{proof}
    Fix a concept $c \in \C$ and a set $A \subseteq \X$ that is independent in $G$ and such that $N[x]$ is $c$-bichromatic for every $x \in A$.
    If $\alpha_2=\alpha_2(G,\C)$ is finite then there exist such $c,A$ with $|A|=\alpha_2$ (and otherwise we can let $A$ be arbitrarily large).
    To simplify the proof we assume all points of $A$ have the same $c$-label (obviously, at least half of them have).
    For every $x \in A$ fix some $z_x \in N(x)$ with $c(z_x) \ne c(x)$, and let $B=\{z_x:x \in A\}$.
    Clearly $|B| \le |A|$.
    Moreover, for every $x \in A$ let $d_B(x) = |N(x) \cap B|$.
    By construction, $1 \le d_B(x) \le |B| \le |A|$.
    By a binning argument, then, there exists a subset $A' \subseteq A$ with $|A'| \ge \frac{|A|}{\lg_2 |A|}$ and some $d \in \{1,\ldots,|B|/2\}$ such that $d \le d_B(x) \le 2d$ for all $x \in A'$.
    Let then $B' = B \cap \bigcup_{x \in A'} N(x)$.
    Note that for all $x \in A'$ we have $|N(x) \cap B'| = |N(x) \cap B| = d_B(x)$.

    We now use $A'$ and $B'$ as support of a family of distributions.
    We then show that, when we pick a distribution uniformly at random from this family, achieving expected error $\eps$ requires drawing $\Omega(|A'|/\eps)$ samples.
    By Yao's minimax principle this implies that for \emph{every} algorithm there exists \emph{some} distribution from the family that yields the same bound.
    
    To begin with, define $p(z) = \frac{1}{d |A'|+ |B'|}$ to every $z \in B'$, and $p(x) = \frac{d}{d |A'| + |B'|}$ for every $x \in A'$.
    Note that $p(z) = \frac{p(x)}{d}$ for every $z \in B'$ and $x \in A'$.
    Let $A' = \{x_i : i = 1,\ldots,K\}$.
    For simplicity, and without loss of generality, we assume $K$ is even.
    Now, for every string $\bs \in \{-1,+1\}^K$ that sums to zero, define the distribution $\D_{\bs}$ as follows: : $\D_{\bs}(z) = p(z)$, for $z \in B'$ and $\D_{\bs}(x_i) = p(x)(1 + s_i\cdot\sqrt{\eps})$ for $i \in [K]$.
    One can check that $\D_{\bs}$ is indeed a distribution.
    Now let $\bs$ be chosen uniformly at random (again so that it sums to zero).
    Since $\D_{\bs}(A') \ge \frac{1}{2}$, then
    \begin{align}
        \E_{\bs}[\ell_{\D_{\bs}}(\hat{y}_{\D_{\bs}},y_{\D_{\bs}})]
        &\ge \frac{1}{2} \E_{\bs}\E_{S\sim{\D}_{\bs}^m} \left[\frac{1}{K} \sum_{i=1}^K \ell_{x_i}(\hat{y}_{\D_{\bs}},y_{\D_{\bs}})\right]\,.
    \end{align}  

    Now, without loss of generality, we may assume the learner knows the concept $c \in \C$ the sets $A'$ and $B'$, as well as the function $p(\cdot)$, see above; it only ignores $\bs$.
    We may also assume the predictor's output on $x_i \in A'$ is function only of the number of times $M_{\bs,i}^m$ that $x_i$ appears in the training set.
    Clearly, $M_{\bs,i}^m \sim \operatorname{Bin}(m,\D_{\bs}(x_i))$.
    The Kullback-Leibler divergence and Pinsker's inequality then yield, for any two $\bs,\bs'$ as above:
    \begin{align}
        \tvd{M_{\bs,i}^m- M_{\bs',i}^m} \le \O\left(\sqrt{\frac{m}{\D_{\bs}(x_i)}} \cdot \left|\D_{\bs}(x_i) - \D_{\bs'}(x_i)\right| \right) \le \O\left(\sqrt{\frac{m \,\eps}{K}} \right) \;.
    \end{align}
    Moreover, if $\bs,\bs'$ differ on the $i$-th coordinate, then  $|y_{\D_{\bs}}(x_i) - y_{\D_{\bs'}}(x_i)| \ge \Omega(\sqrt{\eps})$. 
    Since $\bs$ is uniformly random, the expected square loss at $x_i$ is therefore at least
    \begin{align}
        \Omega(\eps)\left(1- \tvd{M_{\bs,i}^m- M_{\bs',i}^m} \right) =
        \Omega(\eps)\left(1- \O\left(\sqrt{\frac{m \,\eps}{K}}\right)\right) \;.
    \end{align}
    By averaging over all $x_i$, to have total expected loss at most $\eps$, one needs a sample of size $m = \Omega\left(\frac{K}{\eps}\right) = \Omega\bigl(\frac{\alpha_2}{\eps \log\alpha_2}\bigr)$,
    which concludes the proof.
\end{proof}
Note that the lower bound in \Cref{lem:LB2} applies even if the target concept $c$ is known. The uncertainty comes from the marginal distribution $\D$ that determines the averaged labels $\yD$.

\section{Discussion}\label{sec:discussion}
We provide some comments and discuss multiple interesting extensions of our learning problem.

First note that, despite speaking of a graph $G$ in our exposition, our results also apply on infinite domains which may even be uncountable. In particular, they are applicable in standard geometric settings such as in Euclidean spaces. For example, the neighborhoods can be hyperrectangles in $\mathbb{R}^d$ or balls in a general metric space. In this case, the independence number of the graph $\alpha(G)$ becomes the packing number of the metric space.
Further note that we can easily extend our analysis to other loss functions, which typically only requires an alternative of \Cref{lem:heavyVertex}.
For example, if we care about the absolute loss instead of the squared loss, we have to use $\scO(1/\eps^2)$ samples to estimate the average label in each neighborhood instead of $\scO(1/\eps)$ as before.
The remaining arguments are otherwise independent of the particular loss adopted.

Beyond just predicting averaged labels, there can be situations where also the input sample consists of averaged labels instead of the original labels of the instances.
In particular, there could be a second graph $G_{\mathrm{in}}=(\X,E_{\mathrm{in}})$ that models how the labels are averaged for the training set: the training set $S\subseteq\X$ contains points $x$ labeled by $\yD_{\mathrm{in}}(x) = \E_{x' \sim \D}\bigl[c(x') \mid x' \in N_{G_{\mathrm{in}}}[x]\bigr]$ or alternatively by the empirical average $\yD_{\mathrm{in}}(x) = \frac{1}{|N_{G_{\mathrm{in}}}[x]\cap S|}\sum_{x'\in N_{G_{\mathrm{in}}}[x]\cap S} c(x')$.
Studying the learnability of $(G_{\mathrm{in}},G,\C)$---with input labels $\yD_{\mathrm{in}}$ and target labels $\yD$---is a promising research direction.
Our studied learning problem corresponds to the case $E_{\mathrm{in}}=\emptyset$ but arbitrary $E$, while different variants of the \emph{learning from label proportions} problem correspond to either $E=\emptyset$ but arbitrary $E_{\mathrm{in}}$ (see the mentioned references in \Cref{sec:intro}) or to $E=E_{\mathrm{in}}$ \citep{iyer2016privacy,fish2017complexity}.

Moreover, as noted in \Cref{sec:further}, extending our analysis from estimating $\E[c(x')\mid x'\in N[x]]$ to $\E[\ind{g_x(x')\neq c(x')}\mid x'\in N[x]]$ would be interesting.
In this case, $g_x:N[x]\to \{0,1\}$ is a local classifier from a known (partial) hypothesis class $\mathcal{H}$ (say linear classifiers) defined on each neighborhood.
This setup would encompass our learning problem here (with $g_x$ being constant functions) as well as the auditing framework of \citet{bhattacharjee2024auditing}.

Also, extensions to a regression setting with target labels in $[0,1]$ are possible.
This would allow to formulate estimation problems like \emph{``what is the average income of all people with at least my age?''}, where \emph{income} is the target label and the directed neighborhoods are given by \emph{age}.

Additionally, there might be cases where we want to  assign similarities to each neighbor. Instead of $\E_{x'}[c(x')\mid N[x]]$ we want to estimate $\E_{x'}[w(x,x')c(x')\mid N[x]]$ for a given similarity function $w: \X\times\X \to [0,1]$, e.g., Gaussian. %
Our results correspond to the uniform similarity $w(x,x')=1$.
Learning such weighted averages in the special case of $N[x]=\X$ would bring our setup closer to the smoothed analysis of \citet{chandrasekaran2024smoothed}.

Finally, in \Cref{sec:ERM} we discuss how empirical risk minimization (ERM) can be applied to our problem. In \Cref{sec:graphpairs} we discuss a generalization of our results to a setting where instead of a fixed graph $G$,  graphs $G$ and concepts $c$ come as pairs from a known class $\F\subseteq\G\times\C$.

\bibliography{references}

\clearpage
\appendix
\crefalias{section}{appendix} %
\section{Missing proofs} \label{app:proofs-ub}
This section contains the missing proofs and details which enable our main results from \Cref{sec:main}.
We use the following simple estimation lemma for the squared loss.
\heavyVertexLemma*
\begin{proof}
    For any fixed $x \in \X$, define $f(x)$ to be the empirical mean of $c$-labels over points in $N[x] \cap S$; that is for any $x \in \X$ we can write
    \begin{equation}
        f(x) = \frac{1}{M_x} \sum_{i=1}^m c(x_i)\ind{x_i \in N[x]} \;,
    \end{equation}
    where $M_x = \sum_{i=1}^m \ind{x_i \in N[x]}$ is the number of sampled points in $S$ that belong to $N[x]$.  
    Given this, we know that $h_S(x) = f(x)$ whenever $M_x > 0$, and otherwise $h_S(x)$ is the output returned by the OIG predictor.

    Fix $x \in \X$ to be any point satisfying $\D(N[x]) \ge \lambda$.
    Defining $p_x = \Pr_{x' \sim \D}(x' \in N[x]) = \D(N[x])$, we can notice that $M_x \sim \operatorname{Bin}(m,p_x)$ is a binomial random variable with mean $\E_S[M_x] = mp_x$.
    Consequently, a multiplicative Chernoff bound for binomials shows that
    \begin{equation}
        \Pr_{S}\bigl(M_x \le mp_x/2\bigr)
        \le e^{-mp_x/8} \le e^{-m\lambda/8} \,.
    \end{equation}
    where the last inequality is due to the assumption that $p_x = \D(N[x]) \ge \lambda$.
    On the other hand, we will now bound the expected error, conditioned the event that $M_x$ is not small. Specifically, for a fixed $x$ and a fixed integer $M \ge mp_x/2$, let $\hat{y}(S) = \frac{1}{M}\sum_{i=1}^M c(x_i)$ given a sample $S$ of size $M$, and $y = \yD(x)$. Observe that
    \begin{equation}
        \E_S\bigl[(f(x) - \yD(x))^2 \mid M_x = M\bigr] = \E_{S_x \sim \D_x^M}\bigl[(\hat{y}(S_x) - y)^2\bigr] \;,
    \end{equation}
    where $\D_x$ is the distribution   $\D$ conditioned on $N[x]$.
    Then, notice that the quantity we are considering here is the variance of an empirical mean estimator over i.i.d.\ random variables bounded in $[0,1]$, and so by \Cref{clm:empirical} we have $\E_{S_x \sim \D_x^M}\bigl[(\hat{y}(S_x) - y)^2\bigr] \le 1/(4M)$.
    Therefore,  it follows that 
    \begin{equation}
        \E_S\bigl[\ell_x(h_S) \mid M_x=M\bigr]
        = \E_S\bigl[(f(x) - \yD(x))^2 \mid M_x = M\bigr]
        \le \frac{1}{4M} \le \frac{1}{2mp_x} \le \frac{1}{2m\lambda} \;.
    \end{equation}
    By then taking the law of total expectation for all possible values $M \ge mp_x/2$, we also get that,
        \begin{equation}
        \E_S\bigl[\ell_x(h_S) \mid M_x \ge mp_x/2\bigr] \le \frac{1}{2m\lambda} \;.
    \end{equation}
    Lastly, combining all these observations together, and using the fact that $\ell_x(h_S) \in [0,1]$, we finally derive that
    \begin{align}
        \E_S\bigl[\ell_x(h_S)\bigr]
        &\le \E_{S}\bigl[(f(x) - \yD(x))^2 \mid M_x \ge mp_x/2\bigr] + \Pr_{S}\bigl(M_x \le mp_x/2\bigr) \\
        &\le \frac{1}{2m\lambda} + e^{-m\lambda/8} 
        \le \frac{7}{2m\lambda} \;,
    \end{align}
    which concludes the proof.
\end{proof}

The proof of the above lemma relies on this standard result about the variance of the empirical mean of bounded random variables.
\begin{proposition}\label{clm:empirical}
 For a positive integer $M \in \nat_+$, let $a_1,\dots,a_M$ be i.i.d.\ $[0,1]$-valued random variables with $\E[a_i]=\mu$, and define $\bar a$ to be their average. Then,
\[
    \E\bigl[(\bar a-\mu)^2\bigr] \le \frac{1}{4M} \;.
\]
\end{proposition}
\begin{proof}
Since $a_1,\dots,a_M$ are i.i.d.\ with mean $\mu$, we have $\mathbb{E}[\bar a]=\mu$ and hence $\mathbb{E}\big[(\bar a-\mu)^2\big]=\mathrm{Var}(\bar a)$.  Moreover, by independence of the $a_i$'s we also have $\mathrm{Var}(\bar a) = \mathrm{Var}(a)/M$. 
Since $a \in [0,1]$, we have $\mathrm{Var}(a) \le 1/4$ by Popoviciu's inequality.
Combining the above yields the claim. 
\end{proof}

For the high-probability upper bound in \Cref{thm:highProbUB}, we rely on the following lemma showing how the expected squared distance of the median of bounded random variables from a target random variable improves whenever more than half of them are sufficiently close in expectation.
\begin{lemma}\label{lemma:median}
Let $y_1,\dots,y_k,y$ be random variables and let $\eps \ge 0$. If $p> k/2$ of the $i\in[k]$ satisfy $\E[(y_i-y)^2]\le \eps$ then the median $\ymed=\operatorname{median}(y_1,\dots,y_k)$ satisfies  $\E[(\ymed-y)^2]\le \frac{2\eps}{p/k-1/2}$.
\end{lemma}
\begin{proof}
    We use the notation $(\cdot)_+=\max(0,\cdot)$ and $(\cdot)_-=\min(0,\cdot)$. Note that 
    \begin{equation}\label{eq:posnegdecomposition}
        (y_i-y)^2 = ((y_i-y)_+)^2+ ((y_i-y)_-)^2
    \end{equation}
    holds for all $i$ and the same holds for $\ymed$ instead of $y_i$.
    Let us first assume that $y\le\ymed$. By definition of a median half of the $i\in[k]$ satisfy $y_i\le \ymed$ and half $y_i\ge \ymed$. Call an $i$ \emph{good} if $\E[(y_i-y)^2]\le \eps$ and denote the set of all good indices by $P\subseteq[k]$. As we have $p$ good $i$'s, there must be $p-k/2>0$ good $i$'s such that $y_i\ge\ymed\ge y$. Denote the good indices $i$ satisfying $y_i\ge \ymed$  by $P_+\subseteq P$.
    By the choice of $P_+$, we have
    \begin{align}\label{eq:medToAvg}
        ((\ymed-y)_+)^2
        &\le \frac{1}{|P_+|}\sum_{i\in P_+}((y_i-y)_+)^2 \\
        &\le \frac{1}{|P_+|}\sum_{i\in P_+}((y_i-y)_+)^2 + \frac{1}{|P_+|}\sum_{i\in P\setminus P_+} ((y_i-y)_+)^2 \\
        &= \frac{1}{|P_+|}\sum_{i\in P}((y_i-y)_+)^2 \;.
    \end{align}
    From \Cref{eq:posnegdecomposition} we also have that $((y_i-y)_+)^2 \le (y_i-y)^2$.
    Thus, taking expectations, we see that
    \begin{align}
        \E[((\ymed-y)_+)^2]&\le \frac{1}{|P_+|}\sum_{i\in P}\E[((y_i-y)_+)^2] \\
        &\le \frac{1}{|P_+|}\sum_{i\in P}\E[(y_i-y)^2] \\
        &\le \frac{\eps|P|}{|P_+|}\le \frac{\eps k}{p-k/2} \;.
    \end{align}
    By a symmetric argument the same holds for the case $y\ge \ymed$ (and correspondingly $(\ymed-y)_-$).
    Overall we get
    \begin{equation}
         \E[(\ymed-y)^2] = \E[((\ymed-y)_+)^2] + \E[((\ymed-y)_-)^2]\le \frac{2\eps}{p/k-\frac{1}{2}} \;.
    \end{equation}
\end{proof}

\section{The One-Inclusion Graph algorithm} \label{app:oi}

In this section we describe the classic One-Inclusion Graph algorithm (OIG), that is used in our \Cref{alg:main}.
We start by defining the one-inclusion graph of a concept (or hypothesis) class $\C$, where the idea is to translate a classification learning problem to the language of graphs.

\begin{definition}[One-inclusion graph; \citealp{haussler1994predicting}]
    The \emph{one-inclusion graph} of $\C\subseteq \{0,1\}^n$ is a graph $\mathcal{G}(\C)=(V,E)$ defined as follows.
    The vertex set is $V=\C$, and the edge set $E$ corresponds to all pairs $(v,v') \in V$ such that the Hamming distance between $v,v'$ is exactly $1$; that is, there exists $i \in [n]$ such that $v(i) \neq v'(i)$ and for all $j\neq i$, $v(j) = v'(j)$. 
\end{definition}
We now define an orientation of the graph in the standard way. 

\begin{definition}
    An \emph{orientation} of the graph $(V,E)$ is a mapping $\sigma:E\to V$ such that $\sigma(e)\in e$ for each edge $e\in E$.
\end{definition}

The OIG captures a model for {transduction in machine} learning.
A key observation of \citet{haussler1994predicting} is that this model captures an essential ingredient of general PAC learnability; see also \citet{rubinstein2006shifting,daniely2014optimal, brukhim2022characterization,bressan2025fine}. 
The OIG algorithm is presented in \Cref{algo:one_inc} below. 

\begin{algorithm}[H]
\caption{The one-inclusion algorithm $\A_\C$ for $\C\subseteq\Y^{\X}$} \label{algo:one_inc}
\textbf{Input:} A $\C$-realizable sample $S = \bigl((x_1, y_1),\dots,(x_m, y_m)\bigr)$. \\
\textbf{Output:} A hypothesis $\A_{\C}(S)=f_S:\X \to \Y$. \\
For each $x \in \X$, the value $f_S(x)$ is computed as follows.
\begin{algorithmic}[1]
\STATE Consider the class of all patterns over the \emph{unlabeled data}
{$\C|_{(x_1,\dots,x_m,x)} \subseteq \Y^{m+1}$}.
\STATE Find an orientation $\sigma$ of $\mathcal{G}(\C|_{(x_1,\dots,x_m,x)})$ that minimizes the maximum out-degree.
\STATE Let $e=e_{y_1,\dots,y_m}$ denote the edge determined by the labels in $S$ over the first $m$ coordinates.  

\STATE Set $f_S(x) \coloneq \sigma(e)_{m+1}$. That is, the label $f_S(x)$ is determined by the label in the last coordinate of the node points to by the orientation $\sigma$. 
\end{algorithmic}
\end{algorithm}
 
The algorithm gets as input a realizable training sample $S = ((x_1,y_1),\dots,(x_m,y_m))$ as well as an additional test point $x$.
Its goal is to provide a good prediction for the label of $x$.
An orientation of the graph provides the prediction for the label of $x$.

\section{VC classes allow ERM}\label{sec:ERM}
We briefly note that under stronger assumptions then required by our main theorem (\Cref{thm:UBexpectation}), we can instead use (an appropriate variant of) Empirical Risk Minimization (ERM) to design a $(G,\C)$-learner.
In particular, since the analysis of ERM leverages uniform convergence for VC classes, we additionally require that $\C$ has finite VC dimension $d=\vc(\C)$.

The alternative ERM-based algorithm performs what follows.
Using an  i.i.d.\ sample $S$ from $\D$ of size $\O\bigl(\frac{d\log(1/\eps) + \log(1/\delta)}{\eps}\bigr)$ labeled by any fixed concept $c \in \C$, use ERM over $S$ to obtain a predictor $h_{\mathrm{erm}}$ that satisfies $\Pr_{x \sim \D}(h_{\mathrm{erm}}(x) \neq c(x)) \le \eps/2$ with probability at least $1-\delta/2$; this follows from standard results on realizable PAC learning for VC classes \citep{blumer1989learnability}.

We now use a second sample $S'$ of size $\O\bigl(\frac{\alpha_2(G,\C) \log(1/\eps)}{\eps}\bigr)$ to compute the empirical mean labels over the neighborhood of a test point $x \sim \D$ using both the sampled points within $N[x]$ as well as the label $h_{\mathrm{erm}}(x)$ predicted by the ERM classifier.
Now let $M_x = |N[x] \cap S'|$ be the number of examples from $S'$. %
Then, the resulting predictor $h_{S'}$ outputs
\begin{equation}\label{eq:avgERM}
    h_{S'}(x) = \frac{1}{M_x+1}\biggl(h_{\mathrm{erm}}(x) + \sum_{x' \in N[x] \cap S'} c(x')\biggr)
\end{equation}
for any $x \in \X$.
While we can show that this predictor provides the desired guarantee in expectation, by a confidence amplification argument as in \Cref{thm:highProbUB}, given the success of $h_{\mathrm{erm}}$, we can extend the loss guarantee to one that holds with probability $1-\delta/2$.
We can show that the learning algorithm we just described yields the following sample complexity.
\begin{proposition}
Let $\C\subseteq\{0,1\}^\X$ be a class with VC dimension $d=\vc(\C)$ and $G=(\X,E)$ a directed graph. There exists an ERM-based $(G,\C)$-learner with sample complexity
\[
    m_{\mathrm{erm}}(\eps,\delta)
    = \scO\left(\frac{\bigl(d + \alpha_2(G,\C)\log(1/\delta)\bigr)\log(1/\eps)}{\eps}\right)\,.
\]
\end{proposition}
\begin{proof}
    Fix $c \in \C$ to be the ground-truth labeling.
    As already mentioned, we know that the ERM predictor $h_{\mathrm{erm}}$ guarantees $\Pr_{x \sim \D}(h_{\mathrm{erm}}(x) \neq c(x)) \le \eps/2$ with probability $1-\delta/2$.
    By standard realizable PAC learning results, as mentioned above, this can be done using $\O\bigl(\frac{d\log(1/\eps) + \log(1/\delta)}{\eps}\bigr)$ i.i.d.\ samples.
    We condition on this event in what follows.

    On a mass of at most $\eps/2$ of the points, $h_{\mathrm{erm}}$ errs. For these points we can simply assume the worst-case upper bound of $1$ on the loss, and condition now on the event that $h_{\mathrm{erm}}$ predicts the true label $c(x)$ on the test point $x$.

    Next define $h_{S'}$ as in \Cref{eq:avgERM} using the second sample $S'$. 
    The proof of \Cref{thm:UBexpectation} goes through almost exactly as before. In particular, the cases $\Ev_1$ and $\Ev_2$ remain valid and we can choose the constants in the sample complexity such that the expected error is at most $\eps/2$ in case $\Ev_1$ (and $0$ in case $\Ev_2$). The only difference is here that we average over $M_x+1$ labels in $N[x]$. This is valid due to the following two reasons. First, as we get the correct label $h_{\mathrm{erm}}=c(x)$ for the test point $x\in N[x]$. Second, as $x$ is i.i.d.\ from $\D$ averaging over $(N[x]\cap S')\cup \{x\}$ is equivalent to sampling $M_x+1$ points from $N[x]$ and the analysis proceeds as before (with one additional labeled point). In case $\Ev_3$, where $M_x=0$, it simply holds that $\yD(x)=c(x)=h_{\mathrm{erm}}(x)$ and thus we predict correctly.

    Finally, we again apply \Cref{lemma:highProbUB} to turn the in-expectation guarantee of $h_{S'}$ into one with probability $1-\delta/2$ instead of just in expectation, always given that $h_{\mathrm{erm}}$ succeeds.
    We recall it suffices to independently construct $\O(\log(1/\delta))$ predictors $h_{S'}$ as described above, where each one of them requires $\O\bigl(\frac{\alpha_2(G,\C)\log(1/\eps)}{\eps}\bigr)$ i.i.d.\ samples to provide its individual guarantee.
    
    A union bound over the failure of $h_{\mathrm{erm}}$ and that of the median-based predictor from \Cref{lemma:highProbUB} concludes the proof.
\end{proof}

Note that $d\ge\alpha_1(G,\C)$ by definition, and indeed the gap can be arbitrarily large; e.g., if $G$ is a clique then $d=\vc(\C)$ grows arbitrarily as the class $\C$ becomes more complex while $\alpha_1(G,\C) \le 1$ for any $\C$.
This shows that this simpler ERM-based approach is sufficient for learning if $\C$ is a VC class, but can lead to a significantly worse sample complexity than the one achieved by our \Cref{alg:main}, and proved in \Cref{thm:highProbUB}.

We remark that the question of whether ERM also works with a sample complexity depending on $\alpha_1(G,\C)$ instead of $d=\vc(\C)$ remains open. The next section shows that at least in a more general setting, such ERM-based approaches are not sufficient for learnability.

\section{Extension to multiple graphs}\label{sec:graphpairs}
In a possible generalization of our problem, instead of having a single fixed graph, we can imagine that the concept $c$ and graph $G$ come as a pair $(G,c)\in \F$ from a known joint class $\F\subseteq \G\times\C$. Here $\G$ is a family of graphs all with the same node set $\X$. The target graph $G$ itself is fixed yet unknown by the learner and only accessible through a neighborhood oracle on the sample, which allows to check whether any two points in the sample are adjacent in the graph. This clearly generalizes our setting as it can be recovered by $\G=\{G\}$. We denote by $L_G(h)$ the averaged loss of a classifier $h$ as in \Cref{eq:loss} with neighborhoods and averaged labels $\yD(\cdot)$ given by the graph $G$.

\begin{definition}
Let $\C\subseteq\{0,1\}^\X$, let $\G$ be a family of directed graphs on $\X$, and let $\F\subseteq \G\times\C$. A learning rule $\A$ is an \emph{$\F$-learner} if there exists a sample complexity $m:(0,1)^2\to\nat$ such that for every distribution $\D$ over $\X$, every $(G,c)\in \F$, and every $\eps,\delta\in(0,1)$ the following holds. When given a set $S$ of $m \ge m(\eps,\delta)$  i.i.d.\ examples generated by $\D$ and labeled by $c$, then the learning rule returns a predictor $h_S = \A(S)$ such that $L_G(h_S) \le \eps$ with probability at least $1-\delta$ over $S$. 
\end{definition}

Our lower bounds apply also in this setting after a small modification.
In particular, $\alpha_1(\F)$ is the largest set $I\subseteq\X$ that is shattered by a set $S\subseteq\C$ (in the usual sense) with the additional requirement that each $c\in S$ can be extended to a tuple $(G,c)\in \F$ where $I$ is independent in $G$.
That is, $\alpha_1(\F)$ is the size of a largest subset $I \subseteq \X$ such that $|\{c \cap I : (G,c)\in\F, I \text{ independent in } G\}| = 2^{|I|}$.
We also adapt $\alpha_2(\F)=\sup_{(G,c)\in\F}\alpha_2(G,c)$. We see that we can replace $\alpha_1$ and $\alpha_2$ in the previous lower bounds with these modifications.

Also our algorithmic upper bound remains valid. Indeed, one can check that \Cref{thm:UBexpectation} goes through exactly as before when we replace the fixed graph $G$ with the chosen $(G,c)\in\F$. The reason is that already for the original problem we had no dependence on the full graph and only used neighborhood oracle access.

We thus get the following sample complexity bounds for this generalized problem:

\begin{equation}
\Omega\!\left(\frac{\alpha_1(\F)+\alpha_2(\F)/\log(\alpha_2(\F))+\log(1/\delta)}{\eps}\right) \le m_\F(\eps, \delta) \le \scO\!\left(\frac{\alpha_1(\F)+\alpha_2(\F)\log(1/\eps)}{\eps}\log\frac{1}{\delta}\right) .
\end{equation}

Note that this problem generalizes not only standard PAC learning but also learning with \emph{partial concept classes} \citep{alon2022theory}. In particular, let $\tilde\C\subseteq\{0,1,\star\}^\X$ be a partial concept class. We can encode any partial concept $\tilde c\in\tilde\C$ by a pair $(G,c)$ where $c$ agrees with $\tilde c$ on the support of $\tilde c$ and outside of the support we set $c$ to $1$. The graph $G$ is given by isolated vertices on the support of $\tilde c$ and a clique on the rest. Call the collection of all these pairs $\F$.  It is easy to verify that there exists an $\F$-learner if and only if the partial class $\tilde\C$ is PAC learnable.

As uniform convergence (and thus ERM and proper learning) does not succeed for general partial concept classes, this shows that uniform convergence will fail for the more general problem of learning with pairs $\F\subseteq\G\times\C$ too.

\vskip-1em

\end{document}